\documentclass{article} %
\usepackage{iclr2024_conference,times}

\usepackage{amsmath,amsfonts,bm}

\def\eqref#1{equation~\ref{#1}}

\def\1{\bm{1}}

\DeclareMathAlphabet{\mathsfit}{\encodingdefault}{\sfdefault}{m}{sl}
\SetMathAlphabet{\mathsfit}{bold}{\encodingdefault}{\sfdefault}{bx}{n}

\usepackage[utf8]{inputenc} %
\usepackage[T1]{fontenc}    %
\usepackage{hyperref}       %
\usepackage{url}            %
\usepackage{booktabs}       %
\usepackage{amsfonts}       %
\usepackage{nicefrac}       %
\usepackage{microtype}      %
\usepackage{xcolor}         %
\usepackage[normalem]{ulem}
\usepackage{algorithm2e}

\usepackage{times}
\usepackage{epsfig}
\usepackage{graphicx}
\usepackage{amsmath}
\usepackage{amssymb}
\usepackage{wrapfig}

\usepackage[normalem]{ulem}

\usepackage{times}
\usepackage{epsfig}
\usepackage{graphicx}
\usepackage{amsmath}
\usepackage{amsthm}
\usepackage{amssymb}
\usepackage{bm}
\usepackage{svg}
\usepackage{caption}
\usepackage{subcaption}
\usepackage{multirow}
\usepackage{tikz}
\usetikzlibrary{arrows.meta}
\usetikzlibrary{plotmarks}
\usetikzlibrary{decorations}
\usepackage{pgfplots}
\pgfplotsset{compat=1.18} 
\usepackage{rotating}

\usepackage{standalone}

\newcommand{\methodName}[0]{FairerCLIP}

\newcommand{\review}[1]{#1}

\theoremstyle{definition}
\newtheorem{definition}{Definition}
\newtheorem{theorem}{Theorem}

\newtheorem{lemma}[theorem]{Lemma}

\newtheorem{theorem1}{Theorem}

\newtheorem{lemma1}[theorem1]{Lemma}

\newcommand{\nn}{\nonumber}
\newcommand{\cov}{\mathbb{C}\text{ov}}

\newcommand{\indep}{\perp\!\!\!\perp}
\newcount\gaussF
\edef\gaussR{0}
\edef\gaussA{0}

\makeatletter
\pgfmathdeclarefunction{gaussR}{0}{%
 \global\advance\gaussF by 1\relax
 \ifodd\gaussF
  \pgfmathrnd@%
  \ifdim\pgfmathresult pt=0.0pt\relax%
    \def\pgfmathresult{0.00001}%
  \fi
  \pgfmathln@{\pgfmathresult}%
  \pgfmathmultiply@{-2}{\pgfmathresult}%
  \pgfmathsqrt@{\pgfmathresult}%
  \global\let\gaussR=\pgfmathresult%
  \pgfmathrnd@%
  \pgfmathmultiply@{360}{\pgfmathresult}%
  \global\let\gaussA=\pgfmathresult%
  \pgfmathcos@{\pgfmathresult}%
  \pgfmathmultiply@{\pgfmathresult}{\gaussR}%
 \else
  \pgfmathsin@{\gaussA}%
  \pgfmathmultiply@{\gaussR}{\pgfmathresult}%
 \fi
}

\pgfmathdeclarefunction{invgauss}{2}{%
  \pgfmathln{#1}%
  \pgfmathmultiply@{\pgfmathresult}{-2}%
  \pgfmathsqrt@{\pgfmathresult}%
  \let\@radius=\pgfmathresult%
  \pgfmathmultiply{6.28318531}{#2}%
  \pgfmathdeg@{\pgfmathresult}%
  \pgfmathcos@{\pgfmathresult}%
  \pgfmathmultiply@{\pgfmathresult}{\@radius}%
}

\pgfmathdeclarefunction{randnormal}{0}{%
  \pgfmathrnd@
  \ifdim\pgfmathresult pt=0.0pt\relax%
    \def\pgfmathresult{0.00001}%
  \fi%
  \let\@tmp=\pgfmathresult%
  \pgfmathrnd@%
  \ifdim\pgfmathresult pt=0.0pt\relax%
    \def\pgfmathresult{0.00001}%
  \fi
  \pgfmathinvgauss@{\pgfmathresult}{\@tmp}%
}

\usepackage[capitalize]{cleveref}
\crefname{section}{Sec.}{Secs.}
\Crefname{section}{Section}{Sections}
\Crefname{table}{Table}{Tables}
\crefname{table}{Tab.}{Tabs.}
\crefname{appendix}{App.}{Apps.}
\Crefname{appendix}{Appendix}{Appendices}
\Crefname{lemma}{lemma}{lemma}

\newcommand{\Cross}{$\mathbin{\tikz [x=1.4ex,y=1.4ex,line width=.2ex, red] \draw (0,0) -- (1,1) (0,1) -- (1,0);}$}%

\title{\review{\methodName{}: Debiasing CLIP's Zero-Shot Predictions using Functions in RKHSs}}

\iclrfinalcopy %

\author{%
  Sepehr Dehdashtian\textsuperscript{\normalfont *} \quad Lan Wang\thanks{Equal Contribution.} \quad Vishnu Naresh Boddeti \\
  Michigan State University\\
  \texttt{\{sepehr, wanglan3, vishnu\}@msu.edu}
}

\begin{document}

\maketitle

\begin{abstract}
   Large pre-trained vision-language models \review{such as CLIP} provide compact and general-purpose representations of text and images that are demonstrably effective across multiple downstream \review{zero-shot prediction} tasks. However, owing to the nature of their training process, these models have the potential to 1) propagate or amplify societal biases in the training data and 2) learn to rely on spurious features. This paper proposes \methodName{}, a general approach for making zero-shot predictions of \review{CLIP} more fair and robust to spurious correlations. We formulate the problem of jointly debiasing \review{CLIP's} image and text representations in reproducing kernel Hilbert spaces (RKHSs), which affords multiple benefits: 1) \emph{Flexibility:} Unlike existing approaches, which are specialized to either learn with or without ground-truth labels, \methodName{} is adaptable to learning in both scenarios. 2) \emph{Ease of Optimization:} \methodName{}{} lends itself to an iterative optimization involving closed-form solvers, which leads to $4\times$-$10\times$ faster training than the existing methods. 3) \emph{Sample Efficiency:} Under sample-limited conditions, \methodName\ significantly outperforms baselines when they fail entirely. And, 4) \emph{Performance:} Empirically, \methodName{} achieves appreciable accuracy gains on benchmark fairness and spurious correlation datasets over their respective baselines.
\end{abstract}

\vspace{-0.5em}
\section{Introduction\label{sec:introduction}}
\vspace{-0.5em}
Vision-Language Models \review{such as CLIP}~\citep{radford2021learning} are trained on large-scale datasets of image-text pairs to learn representations that have high similarity for related image-text pairs. While these models have gained significant attention in recent years due to their remarkable zero-shot classification capabilities, they are not flawless. There is growing evidence that \review{such models}  suffer from biases w.r.t. demographic (e.g., sex or skin tone) attributes~\citep{agarwal2021evaluating, wang2021gender, birhane2023hate, birhane2023into, dehdashtian2024utilityfairness} and even non-demographic(e.g., image background or illumination) attributes ~\citep{du2022learning, zhang2022contrastive}.

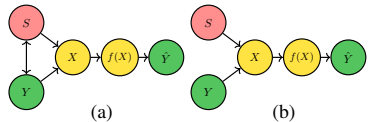
\begin{wrapfigure}[6]{r}{0.4\textwidth}
    \centering
    \vspace{-1.5em}
    \definecolor{color2}{HTML}{f7f8fa} %
\definecolor{color1}{HTML}{DEDEFF} %
\definecolor{color3}{HTML}{FFE342} %
\definecolor{color4}{HTML}{FF9933} %
\definecolor{color5}{HTML}{54c45e} %
\definecolor{color6}{HTML}{ff8f8f} %
\definecolor{color7}{HTML}{CFD3D8} %

\begin{tikzpicture}[node distance=4e]
     \def\lw {1.0}

    \node[draw=black, fill=color3, align=center, anchor=center, circle, line width=\lw pt, minimum width=3em] (x) {$X$};
    \node[draw=black, fill=color3, align=center, anchor=center, circle, line width=\lw pt, minimum width=3em] (f) at ([xshift=2.5em, yshift=0em]x.east) {\small{$f(X)$}};
    \node[draw=black, fill=color6, align=center, anchor=center, circle, line width=\lw pt, minimum width=3em] (s) at ([xshift=-4em, yshift=3em]x) {$S$};
    \node[draw=black, fill=color5, align=center, anchor=center, circle, line width=\lw pt, minimum width=3em] (y) at ([xshift=-4em, yshift=-3em]x) {$Y$};
    \node[draw=black, fill=color5, align=center, anchor=center, circle, line width=\lw pt, minimum width=3em] (yhat) at ([xshift=4em, yshift=0em]f) {$\hat{Y}$};

    \draw[->, line width=\lw pt] (s) -- (x);
    \draw[->, line width=\lw pt] (y) -- (x);
    \draw[<->, line width=\lw pt] (s) -- (y);
    \draw[->, line width=\lw pt] (x) -- (f);
    \draw[->, line width=\lw pt] (f) -- (yhat);

    \node[draw=black, fill=color3, align=center, anchor=center, circle, line width=\lw pt, minimum width=3em] (x1) at ([xshift=10em]f.east) {$X$};
    \node[draw=black, fill=color3, align=center, anchor=center, circle, line width=\lw pt, minimum width=3em] (f1) at ([xshift=2.5em, yshift=0em]x1.east) {\small{$f(X)$}};
    \node[draw=black, fill=color6, align=center, anchor=center, circle, line width=\lw pt, minimum width=3em] (s1) at ([xshift=-4em, yshift=3em]x1) {$S$};
    \node[draw=black, fill=color5, align=center, anchor=center, circle, line width=\lw pt, minimum width=3em] (y1) at ([xshift=-4em, yshift=-3em]x1) {$Y$};
    \node[draw=black, fill=color5, align=center, anchor=center, circle, line width=\lw pt, minimum width=3em] (yhat1) at ([xshift=4em, yshift=0em]f1) {$\hat{Y}$};

    \draw[->, line width=\lw pt] (s1) -- (x1);
    \draw[->, line width=\lw pt] (y1) -- (x1);
    \draw[->, line width=\lw pt] (x1) -- (f1);
    \draw[->, line width=\lw pt] (f1) -- (yhat1);

    \node[align=center, anchor=north, scale=1.7] at ([xshift=2.5em, yshift=-35]x) {(a)};
    \node[align=center, anchor=north, scale=1.7] at ([xshift=2.5em, yshift=-35]x1) {(b)};
    
  \end{tikzpicture}
    \caption{\footnotesize Dependence graphs for debiasing. \label{fig:causal-graph}}
\end{wrapfigure}
The above-mentioned biases can be viewed based on dependencies between the data attributes. We show these dependencies in \cref{fig:causal-graph}: $X$ is the data (e.g., face images) that depends on some attributes, including $Y$, the target attribute we wish to predict, and $S$, the attribute that leads to bias. The goal of bias mitigation is to ensure that the prediction $\hat{Y}$ is independent of $S$. We group the biases into those arising from two scenarios: (1) $Y$ and $S$ are dependent (\cref{fig:causal-graph} a): for example, \emph{high cheekbones} as $Y$ and \emph{sex} as $S$ since males typically have higher cheekbones than females. We refer to this type of correlation as \textbf{intrinsic dependence}. (2) $Y$ and $S$ are independent (\cref{fig:causal-graph} b): for example \emph{hair color} as $Y$ and \emph{sex} as $S$ since the hair color of a person does not depend on their sex. In this case, we refer to any observed correlation as a \textbf{spurious correlation}.

Several efforts~\citep{zhang2022contrastive, gao2021clip, kumar2022fine, kirichenko2022last, chuang2023debiasing, wortsman2022robust, an2023more, adila2023zero}, have been made to debias \review{zero-shot predictions from CLIP models}. However, they are limited in either one or more respects: (1) \textbf{Type of Bias:} Existing \review{CLIP} debiasing methods only consider spurious correlations in the data (\cref{fig:causal-graph} b) and do not seek to address bias induced by pairs of attributes with intrinsic dependencies (\cref{fig:causal-graph} a), (2) \textbf{Labels for training:} All existing approaches are tailored to train/fine-tune either with (supervised) or without (unsupervised) ground-truth labels and, as such, cannot be employed in both scenarios, (3) \textbf{Efficiency:} Some approaches adopt iterative methods to debias the features. However, they are computationally expensive to \emph{train}, i.e.,  slow to converge, leading to high training latency and many parameters in the debiasing modules, increasing model sizes even further.

\begin{figure}[t]
    \centering
    \definecolor{color2}{HTML}{f7f8fa} %
\definecolor{color1}{HTML}{DEDEFF} %
\definecolor{color3}{HTML}{FFE342} %
\definecolor{color4}{HTML}{FF9933} %
\definecolor{color5}{HTML}{54c45e} %
\definecolor{color6}{HTML}{ff8f8f} %
\definecolor{color7}{HTML}{CFD3D8} %

\begin{tikzpicture}[node distance=4e]
     \def\lw {1.0}

    \node[align=center, anchor=north] (sy_pred) {\includestandalone[width=0.8\linewidth]{figs/ov-sy_pred-compact}};
    \node[align=center, anchor=north, scale=2] (label_sy) at ([yshift=-1em, xshift=0em]sy_pred.south) {(a) Pre-Train};

    \node[align=center, anchor=west] (train) at ([xshift=-3.5em, yshift=0em]sy_pred.east) {\includestandalone[width=0.35\linewidth]{figs/ov-training-vertical}};
    \node[align=center, anchor=west, scale=2] (label_tr) at ([xshift=10.9em, yshift=0em]label_sy.east) {(b) Train};
    
    \node[align=center, anchor=west] (test) at ([xshift=-1em, yshift=0em]train.east) {\includestandalone[width=0.9\linewidth]{figs/ov-inference}};
    \node[align=center, anchor=west, scale=2] (label_te) at ([xshift=12.5em, yshift=0em]label_tr.east) {(c) Inference};

  \end{tikzpicture}
    \caption{Overview of the train and inference phases of \methodName. (a) Shows the label prediction step. When labels are not available for training, \methodName\ uses cosine similarity between the $X_{T}$ and $X_I$, and $X_{TS}$ and $X_I$ to predict the target attributes and sensitive attributes, respectively. (b) Shows the inputs and outputs for \methodName\ in its training stage. \methodName\ uses representation of images and the corresponding text prompts that are constructed by target attribute ($Y$) along with the predicted labels to find the image and text encoders, i.e., $\bm f^*_I(.; \bm \Theta_I)$ and $\bm f^*_T(.; \bm \Theta_T)$. (c) Shows the inference phase of \methodName\ in which we use the trained image and text encoders to generate debiased representations from the ones generated by CLIP. \label{fig:overview}}
    \vspace{-0.5cm}
\end{figure}
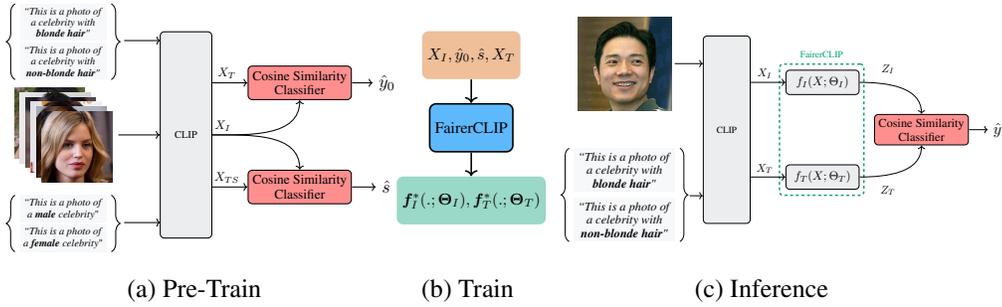

We propose \methodName\ to address the aforementioned limitations of existing debiasing approaches. \methodName\ affords sufficient flexibility to mitigate bias arising from both spurious correlations and intrinsic dependencies and, in both scenarios, learn with or without ground-truth labels. \methodName\ utilizes a non-parametric measure of statistical dependence that accounts for all linear and non-linear relations between the debiased representation and the sensitive attribute of interest. Our formulation lends itself to alternating optimization, with each update having a closed-form solution and, in comparison to baselines, enjoying fast training convergence and requiring fewer parameters to train. An overview of \methodName\ in its train and inference phases along with how we integrate this transformation over the underlying CLIP model is shown in \cref{fig:overview}.

\textbf{Summary of Contributions:} (1) We demonstrate that a single general method can debias the image and text features from \review{frozen CLIP backbones} under different scenarios more effectively than those specialized for each scenario. The scenarios include accounting for both spurious correlations and intrinsic dependencies (\cref{sec:exp:results}), learning with and without ground-truth labels (\cref{sec:exp:results}), and learning from small and medium-sized datasets (\cref{sec:data-size}). (2) We demonstrate that kernel methods are particularly effective compared to shallow MLPs when operating on features and optimizing possibly competing objectives, as is the case for debiasing \review{CLIP} representations. They enjoy closed-form solutions that allow for significantly faster training, can scale to medium-sized datasets, and are more effective under limited training data (\cref{sec:app:runtimes}, \cref{sec:complexity}, and \cref{tab:comp-fairness} (left)).

\vspace{-0.5em}
\section{The Debiasing \review{CLIP} Representations Problem\label{sec:problem-definition}}
\vspace{-0.3em}
\noindent\textbf{Notation:} Scalars are denoted by regular lower case letters, e.g. $r$, $\tau$. Deterministic vectors are denoted by boldface lowercase letters, e.g., $\bm x$, $\bm s$. We denote both scalar-valued and multi-dimensional Random Variables (RVs) by regular uppercase letters, e.g. $X$, $S$. Deterministic matrices are denoted by boldface uppercase letters, e.g. $\bm H$, $\bm \Theta$, and the entry at $i^{th}$ row, $j^{th}$ column of matrix $\bm M$ is denoted by $\left(\bm M \right)_{ij}$ or $m_{ij}$. $\bm I_n$ or simply $\bm I$ denotes an $n\times n$ identity matrix, $\bm 1_n$ or $\bm 1$ and $\bm 0_n$ or $\bm 0$  are $n\times 1$ vector of ones and zeros, respectively. We denote the trace of any square matrix $\bm K$ by $\text{Tr}[\bm K]$. Finite or infinite sets are denoted by calligraphy letters, e.g., $\mathcal H$, $\mathcal A$.

\noindent\textbf{Problem Setup:} We assume that the joint RV $(X_I, X_T, Y, S)$ contains the pre-trained image features $X_I\in \mathbb R^{d}$, pre-trained text features \review{of target attribute} $X_T\in \mathbb R^{d}$, target attribute $Y \in \mathbb{R}^{d_Y}$, and sensitive attribute $S\in \mathbb R^{d_S}$. Their joint distribution will be $\bm p_{X_I, X_T, Y, S}$. Furthermore, $Y$ and $S$ can also belong to any finite set, such as a categorical set. 

Our aim is to debias $X_I$ and $X_T$ by generating representations, $Z_I = \bm f_I(X_I)$ and $Z_T = \bm f_T(X_T)$, with no or reduced \textit{statistical dependence} on $S$. To measure this dependency, we need to employ a metric capable of capturing both linear and non-linear statistical dependencies.

\noindent\textbf{Choice of Dependence Measure:} We will adopt a simplified definition of the Hilbert-Schmidt Independence Criterion (HSIC)~\citep{gretton2005kernel} introduced by \citet{sadeghi2022on}, defined as,
\begin{equation}
\text{Dep}(Z, S) :=\sum_{j=1}^r \sum_{\beta \in \mathcal U_S } \cov^2\left(Z_j, \beta(S)\right),
\label{eq:def-dep}
\end{equation}
where $\mathcal{U}_S$ is a countable orthonormal basis set for the separable universal RKHS $\mathcal{H}_S$, and $Z_j=f_j(X)$ for $f_j \in \mathcal{H}_X \forall j = 1, ..., r$. $\text{Dep}(Z,S)$ can be estimated (see Lemma 1 in \cite{sadeghi2022on}) empirically as,
\begin{equation}\label{eq:dep}
\text{Dep}(Z, S) := \frac{1}{n^2}\left\|\bm \Theta \bm K_X \bm H \bm L_S \right\|^2_F,
\end{equation}
where $n$ is the number of data samples, $\bm K_X \in \mathbb{R}^{n\times n}$ is the Gram matrix corresponding to $\mathcal{H}_X$, $\bm \Theta$ is the encoder parameter in $\bm f(X) = \bm \Theta [ k_{X_1}, k_{X_2}, \cdots, k_{X_n}]^T$, $\bm H = \bm I_n-\frac{1}{n} \bm 1_n \bm 1_n^T$ is the centering matrix, and $\bm L_S$ is a full column-rank matrix corresponding to the Cholesky factorization of $K_{S}$, i.e., $\bm L_S \bm L_S^T=\bm K_{S}$. 
\review{While HSIC and related measures like Maximum Mean Discrepancy (MMD)~\cite{gretton2012kernel} have been employed by prior fairness approaches~\citep{bahng2020learning, quadrianto2019discovering, jung2021fair}, the HSIC variation we use in \cref{eq:dep}} has several attractive properties~\citep{sadeghi2022on}. This includes a convergence rate of $\mathcal{O}(n^{-1/2})$\footnote{\review{In scenarios where only a single or few samples are available, to an extent, heavy data augmentation can compensate for the lack of sufficient samples to accurately estimate $\text{Dep}$. However, this is beyond the scope of this paper, and all our experiments are performed without data augmentation.}}, a practical ability to capture all non-linear modes of dependencies when projecting from a high-dimensional representation to a low-dimensional representation, and, as we demonstrate next, analytical tractability.

In addition to the above-mentioned dependence metric, we also need another metric that can mimic the cosine similarity-based classification employed by \review{CLIP}. This metric will be used to make the representations of images and their corresponding text prompts align with each other to improve the accuracy of the predictions. As a result, we modify the definition of Dep metric in \cref{eq:def-dep} and use a linear kernel as $\beta$ in Lemma~\ref{lemma:emp}.   
\begin{lemma}
\label{lemma:emp}
Let $\bm K_{X_I},\bm K_{X_T}\in \mathbb R^{n\times n}$ be the Gram matrices corresponding to $\mathcal H_{X_I}$ and $\mathcal H_{X_T}$, respectively, i.e., $\left(\bm K_{X_I}\right)_{ij}=k_{X_I}(\bm x_{I_i}, \bm x_{I_j})$ and $\left(\bm K_{X_T}\right)_{ij}=k_{X_T}(\bm x_{T_i}, \bm x_{T_j})$, where covariance is empirically estimated as
\begin{eqnarray}
\cov\left(f_j(X_I), g_m(X_T) \right)\approx \frac{1}{n}\sum_{i=1}^n f_j(\bm x_{I_i}) g_m(\bm x_{T_i})
 -\frac{1}{n^2}\sum_{p=1}^n f_j(\bm x_{I_p}) \sum_{k=1}^n g_m(\bm x_{T_k}).\nn
\end{eqnarray}
It follows that, the corresponding empirical estimator for $\text{Dep}\left(Z_I, Z_T\right)$ is
\begin{eqnarray}\label{eq:empirical-form}
\text{Dep}\left(Z_I, Z_T\right)&=&\frac{1}{n^2}\left\|\bm{\Theta}_I \bm K_{X_I} \bm H \bm K_{X_T} \bm \Theta^T_T \right\|^2_F,
\end{eqnarray}
\end{lemma}
where $\bm \Theta_I$ and  $\bm \Theta_T$ are the parameters of the image and text encoders, respectively, and $\bm K_{X_I}$ and $\bm K_{X_T}$ are the kernel matrices for the image and text features, respectively.
\begin{proof}
The main idea for proving equality~\eqref{eq:empirical-form} is to employ the representer theorem to express $f_j$ and $g_m$. The complete proof is available in the supplementary material.
\end{proof}

\noindent\textbf{Objective Function:} After choosing the appropriate dependence measure, we now define our objective function. 
\review{Our goal is to mitigate bias in CLIP's zero-shot predictions by debiasing the underlying representations}. This can be achieved by (1) reducing the information related to the sensitive attribute while (2) preserving information about the target attribute as much as possible in the pair of image-text representations and (3) keeping the image and corresponding text representations aligned with each other.

We formulate the above-mentioned learning objective through the following optimization problem.
\begin{definition}\label{def:1}
\begin{equation}
\label{eq:main}
\begin{aligned}
\sup_{\bm f_I \in \mathcal A^I_r, \bm f_T \in \mathcal A^T_r} \big\{J\left(\bm f_I, \bm f_T, \tau_I, \tau_T, \tau_z\right)&:=\text{Dep}\left(Z_I, Y\right) - \tau_I\text{Dep}\left(Z_I, S\right) \\
&+ \text{Dep}\left(Z_T, Y\right) - \tau_T\text{Dep}\left(Z_T, S\right) \\
&+ \tau_z\text{Dep}\left(Z_I, Z_T\right)\big\}
\end{aligned}
\end{equation}
where $\mathrm{Dep}(\cdot, \cdot)\geq 0$ is the statistical dependence measure defined in \cref{eq:dep}. $\mathrm{Dep}(Q, U) = 0$ implies $Q$ is independent of $U$ (i.e., $Q \indep U$), and $\mathrm{Dep}(Q, U) > 0$ implies $Q$ is dependent on $U$ (i.e., $Q \not\indep U$), with larger values indicating greater degrees of dependence. $\tau_I$ and $\tau_T$ control the contribution of the corresponding debiasing terms and $\tau_z$ controls the alignment of the debiased image and text features $Z_I=\bm{f}_I(X_I)$ and $Z_T=\bm{f}_T(X_Y)$, respectively.
\end{definition}

In the above definition, the terms $\text{Dep}(Z_I, Y)$ and $\text{Dep}(Z_T, Y)$ contribute to maximizing the statistical dependence between the representations and the target label $Y$, the terms $-\tau_I\text{Dep}(Z_I, S)$ and $-\tau_T\text{Dep}(Z_T, S)$ seek to make the representations independent of $S$, and the term $\tau_z\text{Dep}(Z_I, Z_T)$ ensures that the text and image features are still aligned with each other after debiasing.

\noindent\textbf{Choice of Encoder:}
We construct the mappings through a set of $r$ functions from $\mathbb R^{d_X} \rightarrow \mathbb R$ in a reproducing kernel Hilbert space (RKHS) $\left(\mathcal H_X,\, k_X(\cdot, \cdot)\right)$, such as the RBF Gaussian kernel. Hence, the representation $Z$ can be expressed as,
\begin{eqnarray}\label{eq:f}
 Z = \bm f(X):=\left[ Z_1, \cdots,Z_r \right]^T
\in \mathbb R^r,\quad Z_j=f_j(X),f_j\in \mathcal H_X \ \forall{j=1,\dots, r},
\end{eqnarray}
where $r$ is the dimensionality of the transformed representation.

 Our choice of RKHS is motivated by several reasons. As we observe in \cref{eq:main}, debiasing is inherently an optimization problem with multiple competing objectives. In such cases, optimization is the primary bottleneck rather than model expressivity. This was also observed in~\citet{sadeghi2022on}. The closed-form solution afforded by our approach mitigates the optimization challenges (\cref{sec:app:runtimes} and \cref{sec:app:abl}). RKHS has nice universal approximation properties and has performance comparable to shallow MLPs while being more computationally efficient for training (\cref{sec:app:runtimes}) and is performant under limited data scenarios (\cref{sec:exp:results}).

\vspace{-0.5em}
\section{A Solution to The Debiasing \review{CLIP} Representations Problem\label{sec:approach}}
Given the choice of dependence measure in \cref{eq:dep}, the optimization problem in \cref{eq:main} can be expressed as,
{\small
\begin{equation}
    \label{eq:norm}
    \begin{aligned}
    \max_{\bm f_I \in A_r, \bm f_T \in A_r} \big\{J\left(\bm f_I, \bm f_T, \tau_I, \tau_T, \tau_z, \bm X_I, \bm X_T, \bm Y, \bm S\right)&:= \frac{1}{n^2}\left\|\bm \Theta_I \bm K_{X_I} \bm H \bm L_Y \right\|^2_F - \tau_I \frac{1}{n^2}\left\|\bm \Theta_I \bm K_{X_I} \bm H \bm L_S \right\|^2_F \\
    & \frac{1}{n^2}\left\|\bm \Theta_T \bm K_{X_T} \bm H \bm L_Y \right\|^2_F - \tau_T \frac{1}{n^2}\left\|\bm \Theta_T \bm K_{X_T} \bm H \bm L_S \right\|^2_F \\
    & + \tau_z \frac{1}{n^2}\left\|\bm{\Theta}_I \bm K_{X_I} \bm H \bm K_{X_T} \bm \Theta^T_T \right\|^2_F \big\}
    \end{aligned}
\end{equation}}

\begin{figure}[t]
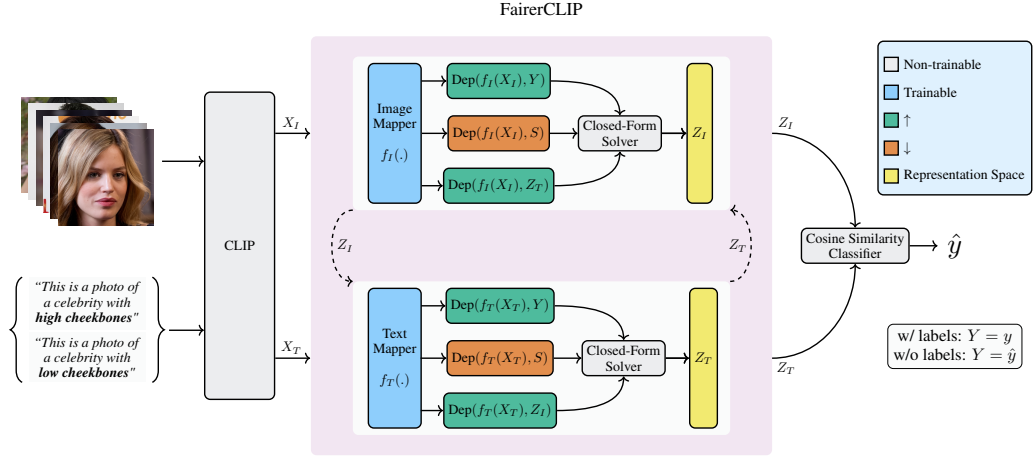

    \centering
    \begin{subfigure}{1.0\textwidth}
        \centering
        \includestandalone[width=\linewidth]{figs/FairVLM-v5}
    \end{subfigure}
    \caption{\methodName\ acts on representations extracted from \review{a frozen CLIP model}. It has two mapping functions, $\bm f_I$ and $\bm f_T$, for the image and text representations. These functions are learned through an alternating optimization algorithm with two closed-form solvers. When ground-truth labels are unavailable for training, \methodName\ learns from pseudo-labels $\hat{y}$, which are initialized from \review{CLIP's zero-shot predictions} and refined iteratively. The bold words in the input text prompts are the information of the target task label included in the text prompts.\label{fig:train-overview}}
    \vspace{-1em}
\end{figure}

Our solution to the constrained optimization problem in \cref{eq:norm} is based on the observation that it has a closed-form solution when either $\bm f_I$ or $\bm f_T$ are fixed. Let $\bm Z_O$ be the feature corresponding to the fixed parameter and $\bm f$ the optimization parameter of the other feature of interest. Then \cref{eq:norm} reduces to two optimization problems of the following general form,
{\footnotesize
\begin{equation}
    \label{eq:fixed}
    \max_{\bm f \in A_r} \{J\left(\bm f, \tau, \tau_z, \bm X, \bm Y, \bm S, \bm Z_O\right):=\frac{1}{n^2}\left\|\bm \Theta \bm K_{X} \bm H \bm L_Y \right\|^2_F - \tau \frac{1}{n^2}\left\|\bm \Theta \bm K_{X} \bm H \bm L_S \right\|^2_F + \tau_z \frac{1}{n^2}\left\|\bm{\Theta} \bm{K}_X \bm H \bm Z_O \right\|^2_F \}
\end{equation}}

This is easy to see since fixing either of the parameters in \cref{eq:norm} renders the terms involving them to a constant w.r.t. the variable of interest, and hence can be ignored during optimization.

\begin{theorem}\label{thm:main}
Let the Cholesky factorization of $\bm K_X$ be $\bm K_X=\bm L_X \bm L_X^T$,  where $\bm L_X\in \mathbb R^{n\times d}$ ($d\le n$) is a full column-rank matrix. Let $r\le d$, then a solution to~\cref{eq:fixed} is
\begin{eqnarray}
\bm f^{\text{opt}}(X) = \bm \Theta^{\text{opt}} \left[k_X(\bm x_1, X),\cdots, k_X(\bm x_n, X)\right]^T\nn,
\end{eqnarray}
where $\bm \Theta^{\text{opt}}=\bm U^T \bm L_X^\dagger$ and the columns of $\bm U$ are eigenvectors corresponding to the $r$ largest eigenvalues of the following generalized eigenvalue problem.
{\footnotesize
\begin{eqnarray}
\label{eq:eigen-f}
\bm L_X^T\left(\bm H\bm K_Y\bm H  - \tau\bm H\bm K_S\bm H + \tau_z\bm H \bm Z_O \bm Z_O^T \bm H\right)\bm{L}_X\bm u = \lambda \left(\frac{1}{n}\,\bm L_X^T \bm H \bm L_X + \gamma \bm I\right) \bm u.
\end{eqnarray}}

Furthermore, the objective value of~(\ref{eq:fixed}) is equal to $\sum_{j=1}^r\lambda_j$, where $\{\lambda_1,\cdots,\lambda_r\}$ are $r$ largest eigenvalues of~\cref{eq:eigen-f}.
\end{theorem}
\begin{proof}
The objective in~\cref{eq:fixed} can be expressed as a trace optimization problem, which reduces to a generalized eigenvalue problem~\citep{kokiopoulou2011trace}. See the supplementary material for detailed proof.
\end{proof}

Building upon the above closed-form solution, we adopt alternating optimization to solve \cref{eq:norm}, by fixing $\bm f_I$ and solving for $\bm f_T$ and vice-versa (\cref{fig:train-overview}). The formulation in \cref{eq:main} requires labels of the downstream target task $Y$ and the sensitive labels $S$ to learn \methodName's parameters. While such labels are readily available for supervised learning \review{and partially available for semi-supervised learning~\citep{jung2022learning, chen2023project}}, this is not the case for \review{unsupervised learning}. Therefore, in this case, we initialize the labels $Y$ and $S$ by the original zero-shot predictions $\hat{Y}$ and $\hat{S}$ from \review{CLIP} (\cref{fig:overview} a). Then we refine $\hat{Y}$ by predicting it after every iterative update of $\bm f_I$ and $\bm f_T$. However, note that we do not update $\hat{\bm{S}}$ in the same way since our initial prediction of $\hat{S}$ has the most information about the label $S$, but as we debias the representations in the subsequent iterations, we remove the information of $S$. Therefore, updated values of $\hat{S}$ will lead to inaccurate estimates of $\text{Dep}(Z, S)$ and affect the overall optimization. This procedure is detailed in \cref{alg:fairvlm-training} of \cref{sec:app:tec-details}.

\vspace{-0.2em}
\subsection{A Geometric Illustration of \methodName{}}
\vspace{-0.2em}
\begin{figure}[ht]
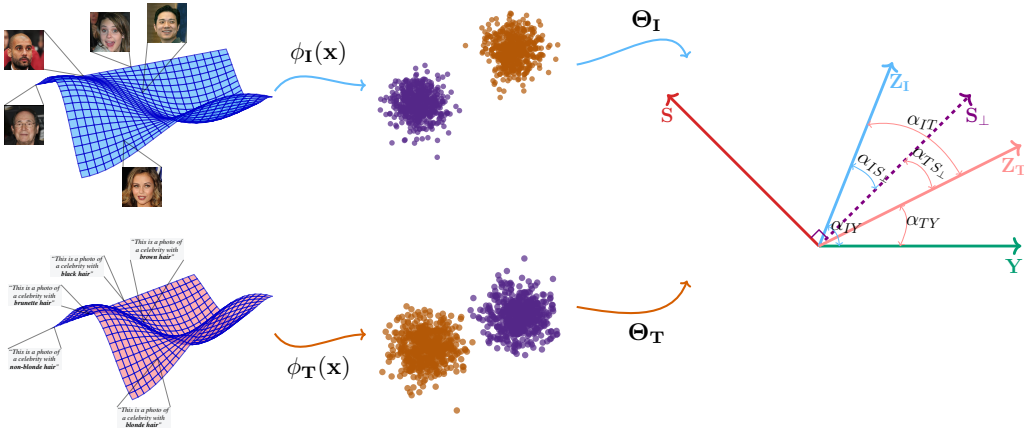

    \centering
    \includestandalone[width=\linewidth]{figs/intuitive}
    \caption{A geometric illustration of \methodName{} training steps. The encoder utilizes the implicit mapping functions $\phi_I(X)$ and $\phi_T(X)$ of the RBF kernel to map image and text features into an infinite-dimensional space, facilitating linear separability of samples with different target attributes. The optimization process seeks a direction that aligns with labels $Y$, statistically independent of $S$, and aligned with the other representation.}
    \label{fig:illustration}
    \vspace{-1em}
\end{figure}

A geometric illustration of the steps that \methodName\ takes to debias the representations is shown in Fig.~\ref{fig:illustration}. In theory, the RBF kernels used in our encoder ($\phi_I(X)$ and $\phi_T(X)$) map the image and text features into an infinite-dimensional space, where the samples corresponding to different target attributes are linearly separable. In the infinite-dimensional space, the encoder that optimizes~\cref{eq:main} for $\bm \Theta_I$ and $\bm \Theta_T$ by alternating between closed-form solvers and seeks a direction for mapping the image and text features that have low angular distance w.r.t. the direction of (1) $Y$ labels (small $\alpha_{IY}$ and $\alpha_{TY}$), (2) $S_{\perp}$ (small $\alpha_{IS_{\perp}}$ and $\alpha_{TS_{\perp}}$), and (3) the other representation (small $\alpha_{IT}$).

\vspace{-1em}
\section{Experimental Evaluation\label{sec:exp}}
\vspace{-0.3em}
We evaluate \methodName\ on datasets with spurious correlation and intrinsic dependence and compare it to several existing baselines. In summary, the experimental results indicate that the baseline methods are effective in mitigating spurious correlations, but they are not as effective at mitigating the bias caused by the intrinsic dependencies. In contrast, \methodName\ effectively and efficiently mitigates both spurious correlations and intrinsic dependencies in \review{CLIP's} zero-shot predictions. In all our experiments, to overcome the $\mathcal{O}(n^3)$ computational and $\mathcal{O}(n^2)$ memory complexity of the kernel matrices $\bm K$, we use random Fourier features (RFF)~\citep{rahimi2007random}. \review{All the implementation details are provided in \cref{sec:imp-details}}

\vspace{-0.1em}
\subsection{Datasets\label{sec:datasets}}
\vspace{-0.1em}
We evaluate \methodName\ on an assortment of \review{classification tasks across many} datasets. This includes \textbf{Waterbirds} \citep{sagawa2019distributionally}, which contains spurious correlations between the types of birds and background of the images, different settings of \textbf{CelebA} \citep{liu2015deep} that contains more than 200,000 face images of the celebrities in the wild annotated with 40 binary attributes and contains both spurious correlations and intrinsic dependencies among its attributes, \textbf{FairFace} dataset \citep{karkkainen2021fairface} which contains more than 108,000 face images from 7 different race groups (White, Black, Indian, East Asian, Southeast Asian, Middle Eastern, and Latino) collected from the YFCC-100M Flickr dataset and labeled with race, sex, and age groups, and \textbf{Chicago Face Database (CFD)}~\citep{ma2015chicago} which includes face images with different annotations such as facial attributes, ethnicity, age, and sex.

\vspace{-0.1em}
\subsection{Empirical Results\label{sec:exp:results}}
\vspace{-0.1em}
We report the results of \methodName\ and compare them with the performance of related baselines on various datasets and settings. Following the experimental settings of prior work~\citep{zhang2022contrastive, koh2021wilds, chuang2023debiasing}, we do not presume sensitive attributes ($S$) during the training process but assume them in the validation dataset for hyperparameter tuning and model selection, as proposed in ~\citet{koh2021wilds}. Thus, following prior work~\citep{zhang2022contrastive}, for \methodName\ and other baselines that need label $S$, we use the zero-shot predictions of $S$ ($\hat{S}$) from \review{CLIP} as the sensitive attribute.

\vspace{-0.1em}
\subsubsection{Mitigating Intrinsic Dependency\label{sec:exp:results:intrinsic}}
\vspace{-0.1em}
To evaluate the ability of \methodName\ to mitigate intrinsic dependency, we conduct numerical evaluations on the CelebA dataset with \emph{high cheekbones} as the target attribute and \emph{sex} as the sensitive attribute. As discussed in Sec.~\ref{sec:introduction}, these two attributes are intrinsically dependent. To measure the fairness of predictions, we employ the Equal Opportunity Difference (EOD) \citep{hardt2016equality} metric, defined as,
    $\text{EOD} := \left|P(\hat{Y} = 1 | Y = 1, S = 1) - P(\hat{Y} = 1 | Y = 1, S = 0)\right|$,  
where $S$ is the sensitive attribute, and $\hat{Y}$ and $Y$ are the predicted and the ground-truth target labels, respectively. Our choice of EOD is justified since other fairness definitions, like Demographic Parity Violation (DPV), are not well suited for many practical scenarios \citep{hardt2016equality, chouldechova2017fair}.

\begin{wraptable}{r}{0.45\linewidth}
\vspace{-0.3cm}
\caption{\small Fairness on the CelebA dataset with intrinsic dependency\label{tab:celeba-intrinsic}.  All values are in \%.}
\vspace{-0.25cm}
\scalebox{0.55}{
\begin{tabular}{lrrcrr}
        \toprule
        \multicolumn{1}{l}{\multirow{2}[4]{*}{Method}} & \multicolumn{2}{c}{CLIP ResNet-50} && \multicolumn{2}{c}{CLIP ViT-L/14} \\
        \cmidrule{2-3} \cmidrule{5-6}         & Avg & EOD && Avg   & EOD \\
        \midrule
        Zero-shot  {\tiny \citep{radford2021learning}} &  50.5     & 5.8 &&   48.8    & \underline{2.8} \\
        ERM Linear Probe {\tiny \citep{kumar2022fine}} &    84.8   &  19.0     &&   84.8    & 14.0 \\
        ERM Adapter {\tiny \citep{gao2021clip}} &   85.3   &    11.0   &&   84.6    & 14.0 \\
        DFR (Subsample) {\tiny \citep{kirichenko2022last}} &    83.2   & 4.2      &&  84.1     & 7.4 \\
        DFR (Upsample) {\tiny \citep{kirichenko2022last}} &   83.6    &   4.1    &&   84.1    & 6.6 \\
        Contrastive Adapter {\tiny \citep{zhang2022contrastive}} & 84.2 & \underline{1.0}  && 83.6 & 6.3 \\
        \methodName\ (ours) &   83.4    &  \textbf{0.02}     && 83.8 & \textbf{0.005} \\
        \bottomrule
        \end{tabular}%
        }
\vspace{-0.35cm}
\end{wraptable}Table~\ref{tab:celeba-intrinsic} compares the performance of \methodName\ and the baselines on the CelebA dataset with intrinsic dependency. For this experiment, we train all methods except the zero-shot baseline, which is included to demonstrate the level of unfairness in the \review{CLIP} features with the ground-truth labels. We observe that among all baselines, Contrastive Adapter~\citep{zhang2022contrastive} performs well and achieves appreciable EOD for the CLIP ResNet-50 model. However, most other methods seem to even amplify the bias in the original \review{CLIP} features while improving average accuracy. \methodName{} performs the best in terms of debiasing, achieving an EOD of 0.002\% and  0.005\% for CLIP ResNet-50 and CLIP ViT-L/14, respectively. Overall, \methodName\ is very effective at mitigating unfairness to a significant extent, achieving an EOD value close to zero while maintaining a high classification accuracy.

\vspace{-0.1em}
\subsubsection{Mitigating Spurious Correlation\label{sec:exp:results:spurious}}
\vspace{-0.1em}

To evaluate \methodName{}'s effectiveness in mitigating spurious correlation, we perform numerical experiments on spurious correlation benchmarks, Waterbirds and CelebA, following the settings in~\citet{zhang2022contrastive}. Since \methodName\ can be learned with or without ground-truth labels, we compare it against methods from these two categories. For performance evaluation, we use three metrics: 1) Average accuracy (\textbf{Avg.}), 2) Worst-Group accuracy (\textbf{WG}), i.e., the lowest accuracy of all subgroups, and 3) \textbf{Gap}, which is the difference between average and worst-group accuracy.

\begin{table}[htbp]
  \centering
  \caption{Comparison of prior methods and \methodName\ in terms of worst group accuracy (WG), average accuracy (Avg), and their gap on the WaterBirds and CelebA datasets. For the latter, the target and sensitive attributes are blonde hair and sex for two different \review{variants of CLIP}, CLIP ResNet-50 and CLIP ViT-L/14, in two different settings--w/ and w/o labels. \textbf{1st} / \underline{2nd} best results are in \textbf{bold} / \underline{underlined}.\label{tab:grouprobust}}
    \resizebox{\columnwidth}{!}{%
    \begin{tabular}{clccccccccccccccc}
    \toprule
    & \multirow{2}[4]{*}{Method / Acc.} & \multicolumn{7}{c}{CLIP ViT-L/14} && \multicolumn{7}{c}{CLIP ResNet-50} \\
    && \multicolumn{3}{c}{Waterbirds} && \multicolumn{3}{c}{CelebA} && \multicolumn{3}{c}{Waterbirds} && \multicolumn{3}{c}{CelebA}\\
    \cmidrule{3-5} \cmidrule{7-9} \cmidrule{11-13} \cmidrule{15-17} && WG ($\uparrow$) & Avg ($\uparrow$) & Gap ($\downarrow$) && WG ($\uparrow$) & Avg ($\uparrow$) & Gap ($\downarrow$) && WG ($\uparrow$) & Avg ($\uparrow$) & Gap ($\downarrow$) && WG ($\uparrow$) & Avg$\uparrow$ & Gap ($\downarrow$) \\
    \midrule
    \multirow{6}{*}{\begin{sideways}w/ labels\end{sideways}} 
    & ERM Linear Probe   {\tiny \citep{kumar2022fine}} & 65.4$\pm$0.5 & \underline{97.7$\pm$0.1} & 32.3$\pm$0.5 &       & 30.4$\pm$1.5 & \textbf{94.6$\pm$0.1} & 64.2$\pm$1.5 &&  13.2$\pm$0.7 & \underline{94.6$\pm$0.1} & 81.4$\pm$0.7 &       & 13.1$\pm$0.9 & \textbf{94.8$\pm$0.0} & 81.6$\pm$0.8\\
    & ERM Adapter{\tiny \citep{gao2021clip}} & 76.1$\pm$1.8 & \textbf{97.8$\pm$0.1} & 21.7$\pm$1.7 &       & 40.0$\pm$5.6 & \underline{94.3$\pm$0.3} & 54.3$\pm$5.6 && 63.0$\pm$0.4& \textbf{96.0$\pm$1.1} & 32.9$\pm$0.8&       & 41.9$\pm$4.5 & \underline{94.7$\pm$0.4} & 52.8$\pm$4.1 \\
    & DFR (Subsample){\tiny \citep{kirichenko2022last}} & 58.8$\pm$0.8 & 95.9$\pm$0.2 & 37.1$\pm$0.8 &       & 78.7$\pm$3.6 & 91.8$\pm$0.2 & 13.1$\pm$3.6 && 66.1$\pm$5.5 & 92.9$\pm$2.2 & 26.9$\pm$6.5 &       & 80.9$\pm$0.6 & 91.7$\pm$0.5 & 10.8$\pm$3.2\\
    & DFR (Upsample){\tiny \citep{kirichenko2022last}} & 66.5$\pm$0.8 & 96.4$\pm$0.9 & 29.8$\pm$1.5 &       & 83.9$\pm$2.3 & 91.2$\pm$0.8 & 7.2$\pm$3.1 && 54.2$\pm$6.2 & 90.3$\pm$2.0 & 36.1$\pm$7.9 &       & \textbf{89.9$\pm$0.2} & 91.3$\pm$0.3 & \textbf{1.4$\pm$0.5}\\
    & Contrastive Adapter {\tiny \citep{zhang2022contrastive}} & \underline{85.3$\pm$2.3} & 94.5$\pm$2.4 & \underline{9.3$\pm$1.1} &       & \underline{83.9$\pm$1.1} & 90.4$\pm$0.2 & \underline{6.4$\pm$1.1} &&  \textbf{82.5$\pm$0.9} & 88.2$\pm$2.6 & \textbf{5.7$\pm$3.1} &       & \underline{88.4$\pm$1.7} & 90.8$\pm$1.2 & \underline{2.5$\pm$1.5}\\
    & \methodName\ (ours) & \textbf{86.0$\pm$1.8} & 92.2$\pm$0.8 & \textbf{6.1$\pm$1.9} &       & \textbf{85.2$\pm$2.3} & 87.8$\pm$1.7 & \textbf{2.5$\pm$0.9} &&   \underline{75.4$\pm$1.9} & 84.3$\pm$2.2 &\underline{ 8.9$\pm$3.1} &       & 81.5$\pm$0.7 & 85$\pm$0.9 & 3.5$\pm$0.3 \\
    \midrule
    \multirow{3}{*}{\begin{sideways} \shortstack[c]{w/o \\ labels}\end{sideways}}
    & Zero-shot{\tiny \citep{radford2021learning}} & 45.3$\pm$0.0   & 84.4$\pm$0.0   & 39.1$\pm$0.0   &       & 72.8$\pm$0.0   & \underline{87.6$\pm$0.0 }  & 14.9$\pm$0.0  && 39.6$\pm$0.0   & 77.3$\pm$0.0   & 37.7$\pm$0.0   &       & 75.9$\pm$0.0   & 82.3$\pm$0.0   & 6.4$\pm$0.0\\
    & Orth-Cali{\tiny \citep{chuang2023debiasing}} & \underline{68.8 $\pm$ 0.0} & 	\underline{84.5 $\pm$ 0.0} & \underline{15.7$\pm$0.0} & & \underline{76.1$\pm$0.0} &	86.2$\pm$0.0 	& \underline{10.1$\pm$0.0} && \underline{74.0$\pm$0.0} 	& \underline{78.7$\pm$0.0} 	& \textbf{4.7$\pm$0.0} & & \textbf{82.2$\pm$0.0 }	& \underline{84.4$\pm$0.0} 	& \textbf{2.2$\pm$0.0} \\
    & \methodName\ (ours) & \textbf{78.1$\pm$1.4} & 	\textbf{85.1$\pm$1.1} &	\textbf{7.1$\pm$2.4} & & \textbf{86.1$\pm$0.8} 	& \textbf{88.0$\pm$1.0} 	& \textbf{1.9$\pm$0.6} && \textbf{74.8$\pm$1.7} &	\textbf{81.4$\pm$0.9} 	& \underline{6.6$\pm$2.5} & & \underline{80.4$\pm$1.0} 	& \textbf{84.7$\pm$0.7} & \underline{4.3$\pm$0.4}\\
    \bottomrule
    \end{tabular}}%
\end{table}%

\cref{tab:grouprobust} shows the results of our empirical evaluation. We make the following observations: (i) On CLIP ViT-L/14, \methodName\ has the lowest Gap and highest WG accuracy. (ii) For the CLIP ResNet-50, \methodName\ outperforms the baselines in the w/o labels setting but not in the w/ label setting. The discrepancy between the performance \methodName\ with CLIP ViT-L/14 and CLIP ResNet-50 can be attributed to the fact that CLIP ResNet-50 features contain less information about target attributes than CLIP ViT-L/14 features, as shown in \cref{sec:app:information}. Overall, the results of \cref{tab:grouprobust} indicate that \methodName\ effectively improves the worst group's accuracy and reduces the Gap. Notably, our approach can be applied to and is effective in both scenarios, with and without ground-truth labels. At the same time, the baselines are specialized to operate in one or the other scenario only.

In \cref{tab:comp-fairness}, we evaluate \methodName\ on the FairFace dataset. Here, we consider sex and race as the sensitive attributes and follow the experimental setup in \citet{chuang2023debiasing}. We use five target attributes and ten text prompts (2 prompts per attribute) unrelated to the samples' facial or sensitive attributes; we do not have access to ground-truth labels. As an example, the text prompt can be \textit{``A photo of a \textbf{criminal} person"} or \textit{``A photo of a \textbf{friendly} person"}. All the ten specific prompts are in \cref{sec:imp-details}. To evaluate the models, we calculate MaxSkew@1000 \citep{geyik2019fairness}, which assesses the maximum imbalance in certain sensitive attributes within a dataset. As is shown in~\cref{tab:comp-fairness} \methodName\ outperforms
\begin{wraptable}{r}{0.45\textwidth}
    \vspace{-1em}
    \centering
    \caption{Comparison of \methodName\ with baselines on FairFace dataset. 
    \label{tab:comp-fairness}\vspace{-0.8em}}
    \resizebox{0.45\columnwidth}{!}{%
      \centering
        \begin{tabular}{lccccc}
            \toprule
            \multicolumn{1}{l}{\multirow{2}[4]{*}{Method / MaxSkew@1000}} &  \multicolumn{2}{c}{CLIP ViT-B/32} &&  \multicolumn{2}{c}{CLIP ViT-L/14} \\
                \cmidrule{2-3}  \cmidrule{5-6} &   Sex   & Race    &&  Sex   & Race \\
            \midrule
            Zero-Shot {\tiny \citep{radford2021learning}}   &  0.206  &  0.743 &&  0.206  &  0.768 \\
            Orth-Proj {\tiny \citep{chuang2023debiasing}}   &  0.146  &  0.755 &&  0.349  &  0.605 \\
            Orth-Cali {\tiny \citep{chuang2023debiasing}}   &  \underline{0.102}  &  \underline{0.638} &&  \underline{0.200}  &  \underline{0.461} \\
            \methodName{} (ours)     &  \textbf{0.097}  &  \textbf{0.408} && \textbf{ 0.099 } &  \textbf{0.428} \\
            \bottomrule
        \end{tabular}
        }
\vspace{-1.5em}
\end{wraptable}
the other baselines for both sensitive attributes across two different \review{CLIP} backbones.

In summary, the results in \cref{tab:grouprobust}, \cref{tab:comp-fairness}, and \cref{tab:celeba-intrinsic} suggest that \methodName\ can effectively mitigate the demographic bias caused by spurious correlation and intrinsic dependency in the data in both w/ and w/o the ground-truth labels settings.

\begin{wrapfigure}{r}{0.35\textwidth}
    \vspace{-1em}
    \centering
    \resizebox{1\linewidth}{!}{
    \centering
    \definecolor{color1}{HTML}{1f77b4} %
\definecolor{color2}{HTML}{ff7f0e} %
\definecolor{color3}{HTML}{2ca02c} %
\definecolor{color4}{HTML}{d62728} %
\definecolor{color5}{HTML}{9467bd} %
\definecolor{color6}{HTML}{8c564b} %
\definecolor{color7}{HTML}{e377c2} %
\definecolor{color8}{HTML}{7f7f7f} %
\definecolor{color9}{HTML}{bcbd22} %
\definecolor{color10}{HTML}{17becf} %

\begin{tikzpicture}
  \begin{axis}[ 
        xlabel=\large{WG (\%)}, 
        ylabel=\large{Avg (\%)},  
        xmin=-2, xmax=45, ymin=45, ymax=65, grid=major, grid style={dashed, line width=1pt, color=black!50}, legend style={at={(.99,0.01)},anchor=south east, font=\small, minimum width=2cm, minimum height=2em},
        xtick distance=10, ylabel near ticks,
        yticklabel style={/pgf/number format/fixed}, xticklabel style={/pgf/number format/fixed}
        ]
    
    \addplot[only marks, mark=triangle*, color6, mark size=7pt] coordinates {(0.0,55.2)};
    \addplot[only marks, mark=triangle*, color1, mark size=7pt] coordinates {(0.0,46.2)};
    \addplot[only marks, mark=triangle*, color5, mark size=7pt] coordinates {(9.3,63.0)};
    \addplot[only marks, mark=*, color4, mark size=7pt] coordinates {(41.4,63.2)};

    \legend{Zero-Shot {\small \citep{radford2021learning}}, ERM Adapter {\small \citep{gao2021clip}}, Contrastive Adapter {\small \citep{zhang2022contrastive}}, \textbf{\methodName\ (ours)}}
  \end{axis}
\end{tikzpicture}
    }
    \caption{Results of \methodName\ and baselines on CFD\label{fig:cfd}}
    \vspace{-1.5em}
\end{wrapfigure}
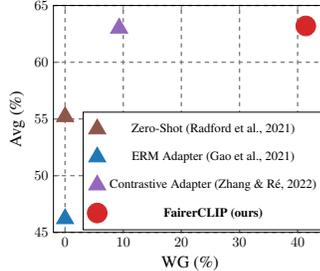
Next, we consider a more challenging task to evaluate the data-efficiency of \methodName. We use CFD images with \emph{attractive} and \emph{sex} as the target and sensitive group attributes. The former is a continuous label, which we binarize by using the mean value of all samples as a threshold. Moreover, the sex attribute is a binary label. This task presents challenges in two aspects. First, the number of samples in this dataset is very small (597 samples), which may not be sufficient for training some of the baselines. Second, the performance of the zero-shot classifier for this case shows that the features generated by the CLIP model are not well separated, rendering it difficult to correctly predict $\hat{S}$ (see \Cref{sec:app:cfd}). \cref{fig:cfd} shows the results of this experiment. We first observe that all the baselines almost completely fail at mitigating the bias for the worst group. In contrast, \methodName's performance is satisfyingly better, both in terms of the worst group (WG) and average (Avg) accuracy. Furthermore, the Gap is significantly lower (21.8\% vs 53.7\% for \citep{zhang2022contrastive}).

\vspace{-0.1em}
\subsection{Computational Efficiency of Training\label{sec:app:runtimes}}
\vspace{-0.1em}
\begin{wraptable}{r}{0.45\textwidth}
    \centering
    \vspace{-0.4cm}
    \caption{\small Training time comparison (in seconds).\label{tab:runtime}}
    \vspace{-0.1cm}
    \scalebox{0.6}{
    \begin{tabular}{lrcr}
        \toprule
        Method & Waterbirds && CelebA \\
        \midrule
        Contrastive Adapter\citep{zhang2022contrastive} &  1202  &&  20602 \\
        ERM Linear Probe\citep{kumar2022fine} &    157    && 2437  \\
        ERM Adapter\citep{gao2021clip} &   161   &&  1924 \\
        DFR (Subsample)\citep{kirichenko2022last} &  128   && 1878 \\
        DFR (Upsample)\citep{kirichenko2022last} &  176   && 2662 \\
        \methodName{} (ours) &  \textbf{32}    &&  \textbf{222}  \\
        \bottomrule
    \end{tabular}}
    \vspace{-1.5em}
\end{wraptable}
To show the computational efficiency of \methodName\, we report and compare the training time of \methodName\ and other baselines in \cref{tab:runtime}. The results show that \methodName\ is an order of magnitude faster than most baselines and almost two orders faster than Contrastive Adapter \citep{zhang2022contrastive}. The underlying model for this experiment is CLIP ViT-L/14, and all the numbers are measured on the same machine.

\vspace{-1em}
\section{Ablation Studies\label{sec:ablation}}
\vspace{-0.3em}
We conduct systematic ablation studies in different settings to investigate the effectiveness of individual components of our approach. The settings include spurious correlation, intrinsic dependency experiments, and scenarios where ground-truth labels are unavailable. The results are shown in \cref{tab:ablation}, where \cref{tab:ablation} (left) shows results of training w/ labels and \cref{tab:ablation} (right) shows training w/o labels. In the following, we describe each of these studies. For more ablation studies please refer to ~\cref{sec:app:abl}.

\textbf{Effect of} $\text{Dep}{(Z, Y)}$ \textbf{term}: Here we study the effect of $\text{Dep}{(Z, Y)}$ by only retaining $\text{Dep}{(Z, Y)}$ in the objective. In this case, the worst group accuracy dropped by 14.5\% for the group robustness experiment, and the EO increased to 0.195\% for the fairness experiment. Although the new features were still able to maintain good separation w.r.t. $Y$, they lost their debiasing ability to a large extent.

\textbf{Effect of} $\text{Dep}(Z_I, Z_T)$ \textbf{term:} We remove $\text{Dep}(Z_I, Z_T)$ to investigate its effect on the alignment between the image and text embeddings. Results show that while maintaining a similar worst group accuracy, there is a decrease in average accuracy for both under the w/ and w/o ground-truth setups. This result demonstrates the contribution of this component in improving the predictions. Similarly, the results for the fairness experiment show that $\text{Dep}(Z_I, Z_T)$ also aids in enhancing the debiasing ability of \methodName.

\textbf{Effect of updating $\hat{Y}$:} In this experiment, we predict $\hat{Y}$ once and fix it through the training process. Updating it during the training iterations has a considerable impact on worst group accuracy. The initial zero-shot accuracy for this group was $72.8\%$. Using the same initial $\hat{Y}$ during training improves the accuracy to $81.1\%$ while updating $\hat{Y}$ while training improves it further to $86.1\%$.

\begin{table}[t]
    \vspace{-1em}
    \caption{Ablation study w/ (left) and w/o (right) ground-truth labels on the CelebA dataset for different target attributes with sex as the sensitive attribute. We compare the effect of different components and parameters of \methodName\ on its performance. Both $\text{Dep}(Z_I, Z_T)$ and $\text{Dep}(Z, Y)$ prove to be necessary and effective in maximizing the metrics. All values are in \%.\label{tab:ablation}} 
    \begin{subtable}[h]{0.6\textwidth}
        \centering
        \scalebox{0.78}{
        \begin{tabular}{lrrrcrr}
        \toprule
        \multicolumn{1}{l}{\multirow{2}[4]{*}{Method}}  & \multicolumn{3}{c}{Blonde Hair} && \multicolumn{2}{c}{High cheekbones} \\
        \cmidrule{2-4} \cmidrule{6-7}                & \multicolumn{1}{c}{WG} & \multicolumn{1}{c}{Avg} & \multicolumn{1}{c}{Gap} && Avg   & EOD \\
        \midrule
        $\text{Dep}(Z, Y)$  only ($\tau=0$)    &   72.2   &   89.8     &  17.6     &&   83.8    & 6.4 \\
         w/o $\text{Dep}(Z_I, Z_T)$ ($\tau_{z}=0$)     &   87.0    &   88.7    &  1.7     &&   83.8    & 0.2 \\
        \methodName    &  86.7     &  89.3     &   2.6    &&       83.8    & 0.005  \\
        \bottomrule
        \end{tabular}}  
    \end{subtable}
    \begin{subtable}[h]{0.35\textwidth}
        \centering
        \scalebox{0.69}{
        \begin{tabular}{lrrr}
        \toprule
        \multicolumn{1}{l}{\multirow{2}[4]{*}{Method}}  & \multicolumn{3}{c}{Blonde Hair} \\
            \cmidrule{2-4}       & WG    & Avg   & Gap \\
        \midrule
        $\text{Dep}(Z, Y)$ only ($\tau=0$)   & 75.0    &  87.7     & 12.7       \\
         w/o $\text{Dep}(Z_I, Z_T)$ ($\tau_{z}=0$)    &   81.8    &   86.1    & 4.3 \\
         w/o updating $\hat{y}$     &    81.1   &   87.3    & 6.2  \\
        \methodName   &   86.1    &     88.8  & 2.7 \\
        \bottomrule
        \end{tabular}}
    \end{subtable}
    \vspace{-0.2cm}
\end{table}

\vspace{-1em}
\section{Related Work\label{sec:related-work}}
\vspace{-0.3em}
\textbf{\review{CLIP} and Bias:} Recent advancements in \review{CLIP like models} utilize multimodal data to learn representations that demonstrably generalize well to many downstream tasks and associated datasets~\citep{radford2021learning,desai2021virtex,singh2022flava,zellers2021merlot,zhang2021vinvl,alayrac2022flamingo}. \citet{radford2021learning} demonstrated that utilizing a simple pretraining task with massive amounts of image-text pairs collected from the Internet, lead to models with strong transferability on different downstream tasks. FLAVA~\citep{singh2022flava} learns representations by jointly pretraining on both unimodal and multimodal data. Flamingo~\citet{alayrac2022flamingo} demonstrated excellent generalization performance in few-shot and zero-shot scenarios. However, growing evidence~\citep{birhane2023hate, birhane2023into} shows these models suffer from spurious correlations and bias towards certain demographic groups. For instance, \citet{chuang2023debiasing} showed that textual prompt embeddings capture spurious correlations. In addition, \citet{agarwal2021evaluating} discovered that zero-shot prediction from \review{CLIP} representations showed a high misclassification rate for certain races. Similarly, \citet{wolfe2022evidence} observed that CLIP embeddings exhibit stereotypes about sex and race. \citet{dehdashtian2024utilityfairness} numerically characterize two near-optimal accuracy-fairness trade-offs and evaluate how far CLIP models are from them. Complementing these observations, we observe that \review{CLIP} exhibits high levels of demographic bias on the CFD and CelebA datasets.

\textbf{Debiasing \review{CLIP}:} Several approaches have been proposed to debias \review{CLIP} embeddings. \citet{wang2021gender} addressed bias in image search by combining balanced sampling and pruning spuriously correlated embeddings. \citet{wang2022fairclip} proposed a two-stage method that used learnable word vector prefixes and a re-representation Matrix for debiasing image retrieval problems. \citet{berg2022prompt} jointly trained an adversarial classifier and image-text contrastive loss, effectively reducing different bias measures. \citet{zhang2022contrastive} employed a contrastive adapter training strategy to enhance group robustness. Following the group robustness evaluation, \citet{chuang2023debiasing} proposed to remove bias from text embeddings by projecting out the biased direction with text data only. \citet{seth2023dear} adapted an additive residual learner module that separates the protected attribute information from the image representation generated by the visual encoder of \review{CLIP}.

\vspace{-1em}
\section{Concluding Remarks\label{sec:conclusion}}
\vspace{-0.3em}
This paper proposed \methodName\ to mitigate bias in zero-shot predictions \review{from CLIP}. It is versatile enough to mitigate bias caused by both spurious correlations and intrinsic dependencies in data and can be trained with or without ground-truth labels. Our key idea was to model the \review{CLIP} debiasing problem in reproducing kernel Hilbert spaces and employ a non-parametric statistical dependence measure that considers all linear and non-linear relations between the representation and the attribute of interest. Our solution in the form of an alternating optimization algorithm is effective across a diverse set of datasets, including Waterbirds, CelebA, FairFace, and the Chicago Face Database, spanning a variety of intrinsic dependencies and spurious correlations among attributes. Lastly, kernel-based approaches are underrepresented in current learning solutions, and \methodName\ shows their strong potential for the type of task considered in this paper due to its flexibility, ease of optimization, and promising performance.

\noindent\textbf{Acknowledgements:} This work was supported in part by the National Science Foundation (award \#2147116) and the Office of Naval Research (award \#N00014-23-1-2417).

{\small
\bibliographystyle{iclr2024_conference}
\bibliography{egbib}
}

\newpage

\appendix

\section{Appendix \label{sec:appendix}}
In our main paper, we proposed \methodName\ to debias the text and image features from pre-trained vision-language models. Here, we provide some additional analysis to support our main results. The supplementary material is structured as follows:

\begin{enumerate}
    \item Representation Disentanglement and Training Algorithm in \cref{sec:app:tec-details}
    \item Proofs Lemmas and Theorems in \cref{sec:proofs}
    \item Implementation details in \cref{sec:imp-details}
    \item Analysis of analytical computational complexity and memory complexity in \cref{sec:complexity}
    \item Effect of data size on the performance of \methodName\ in \cref{sec:data-size}
    \item More ablation studies in \cref{sec:app:abl}
    \item Comparison of CLIP ViT-L/14 and CLIP ResNet-50 in \cref{sec:app:information}
    \item Comparison of more than 100 Zero-Shot CLIP models on CFD in \cref{sec:app:cfd}
\end{enumerate}

\subsection{Techniqual Details \label{sec:app:tec-details}}
\textbf{Representation Disentanglement:} To ensure that the mapping functions avoid learning representations with redundant information where different dimensions are highly correlated with each other, we seek a compact~\citep{bengio2013representation} debiased embedding space. Therefore, we impose additional constraints on the representation. Specifically, we constrain the search space of the mapping functions $\bm f(\cdot)$ to learn a disentangled representation~\citep{bengio2013representation} as follows
\begin{eqnarray}\label{eq:A}
\mathcal A_r:=\left\{\left(f_1,\cdots, f_r\right) \,\big|\, f_i, f_j\in \mathcal H_X,\, \cov \left(f_i(X), f_j(X) \right) +\gamma\, \langle f_i, f_j\rangle_{\mathcal H_X}=\delta_{i,j}\right\}.
\end{eqnarray}
In the above set, $\cov \left(f_i(X), f_j(X) \right)$ part enforces the covariance matrix of $Z=\bm f(X)$ to be an identity matrix and encourages the variance of each entry of $Z$ to be one and different entries of $Z$ to be uncorrelated with each other. The regularization part, $\gamma\,\langle f_i, f_j\rangle_{\mathcal H_X}$ encourages the components of the mapping functions to be of unit norm and as orthogonal as possible to each other. These constraints also aid with numerical stability during empirical estimation~\citep{fukumizu2007statistical}.

\textbf{Training Algorithm:} 
Details of training \methodName\ are presented in Algorithm~\ref{alg:fairvlm-training}. \methodName\ uses representation of images, representation of texts corresponding to target attribute labels ($X_{T}$), and representation of text corresponding to sensitive attribute labels ($X_{TS}$) as its inputs. \methodName\'s goal is to find the image encoder ($\bm f^*_I$) and text encoder ($\bm f^*_T$) that can map the biased features generated by the \review{CLIP} to a debiased representation space. The training algorithm starts with initializing the label predictions. Since this algorithm is used for scenarios where we do not have access to the ground-truth labels of target attributes and sensitive attributes, we need to predict them by zero-shot classification from \review{CLIP} features. However, in scenarios where we have access to the true labels of the target attribute, we can skip the pseudo $Y$ initialization step and use the ground truth $Y$ instead. In the last step of initialization, we need to initialize the representation of images since we are using an alternating method to optimize the parameters of both the image encoder and text encoder. After the initialization step, we start to train both models alternatingly. After each optimization iteration, we update our prediction of $Y$ labels. After reaching the stop condition, the training process is complete.

\RestyleAlgo{ruled}
\SetKwComment{Comment}{/* }{ */}
\begin{algorithm}[ht]
\caption{\methodName{} Training Without Labels\label{alg:fairvlm-training}}
{\small \textbf{Input:} $\bm X_I \in \mathbb{R}^{n\times d}$, $\bm X_{T} \in \mathbb{R}^{|Y|\times d}$, $\bm X_{TS} \in \mathbb{R}^{|S|\times d}$, $m \in \mathbb{N}$} \\
{\small \textbf{Output:} $\bm{f}^*_I$, $\bm{f}^*_T$} \\ 
{\small \textbf{Initialize:}} \\
{\small $\hat{\bm Y}^{(0)} \gets \Big\{\forall \bm x_I \in \bm X_I, \bm x_{T} \in \bm X_{T} \Big| \underset{\bm x_{T}}{\text{argmax}} \frac{\bm x_I^T\bm x_{T}}{\| \bm x_I \| \| \bm x_{T} \| }\Big\} $\Comment*[l]{initialize pseudo Y}}
{\small $\hat{\bm S} \gets \Big\{\forall \bm x_I \in \bm X_I, \bm x_{TS} \in \bm X_{TS} \Big| \underset{\bm x_{TS}}{\text{argmax}} \frac{\bm x_I^T\bm x_{TS}}{\| \bm x_I \| \| \bm x_{TS} \| }\Big\} $\Comment*[l]{initialize pseudo S}}
{\small $\bm Z_I^{(0)} \gets \{\bm f^*_I(\bm X_I) | \bm f^*_I \leftarrow \underset{\bm f_I}{\text{argmax}} \mbox{ } J(\bm f_I, \tau_I, 0, \bm X_I, \hat{\bm Y}^{(0)}, \hat{\bm S}, \bm 0)\}$\Comment*[l]{\cref{eq:fixed}}}

{\small $i \gets 0$\;}

\While{$i < m$}{
{\small $\bm Z^{(i+1)}_{T} \gets \{\bm f^*_T(\bm X_{T}) | \bm f^*_T \leftarrow \underset{\bm f_T}{\text{argmax}} \mbox{ } J(\bm f_T, \tau_T, \tau_z, \bm X_{T}, \hat{\bm Y}^{(i)}, \hat{\bm S}, \bm Z_I^{(i)})\}$\Comment*[l]{solve \cref{eq:fixed}}}

{\small $\bm Z_I^{(i+1)} \gets \{\bm f^*_I(\bm X_I) | \bm f^*_I \leftarrow \underset{\bm f_I}{\text{argmax}} \mbox{ } J(\bm f_I, \tau_I, \tau_z, \bm X_I, \hat{\bm Y}^{(i)}, \hat{\bm S}, \bm Z_{T}^{(i)})\}$\Comment*[l]{solve \cref{eq:fixed}}}

{\small $\hat{\bm Y}^{(i+1)} \gets \Big\{\forall \bm z_I \in \bm Z^{(i+1)}_I, \bm z_{T} \in \bm Z^{(i+1)}_{T} \Big| \underset{\bm z_{T}}{\text{argmax}} \frac{\bm z_I^T \bm z_{T}}{\| \bm z_I \| \| \bm z_{T} \| }\Big\} $\Comment*[l]{refine pseudo Y}}
  
{\small $i \gets i + 1$}
}
\end{algorithm}

\subsection{Proofs \label{sec:proofs}}
\begin{lemma1}
Let $\bm K_{X_I},\bm K_{X_T}\in \mathbb R^{n\times n}$ be the Gram matrices corresponding to $\mathcal H_{X_I}$ and $\mathcal H_{X_T}$, respectively, i.e., $\left(\bm K_{X_I}\right)_{ij}=k_{X_I}(\bm x_{Ii}, \bm x_{Ij})$ and $\left(\bm K_{X_T}\right)_{ij}=k_S(\bm x_{Ti}, \bm x_{Tj})$, where covariance is empirically estimated as
\begin{eqnarray}
\cov\left(f_j(X_I), g_m(X_T) \right)\approx \frac{1}{n}\sum_{i=1}^n f_j(\bm x_{I_i}) g_m(\bm x_{T_i})
 -\frac{1}{n^2}\sum_{p=1}^n f_j(\bm x_{I_p}) \sum_{k=1}^n g_m(\bm x_{T_k}).\nn
\end{eqnarray}
It follows that, the corresponding empirical estimator for $\text{Dep}\left(Z_I, Z_T\right)$ is
\begin{eqnarray}
\text{Dep}\left(Z_I, Z_T\right)&=&\frac{1}{n^2}\left\|\bm{\Theta}_I \bm K_{X_I} \bm H \bm K_{X_T} \bm \Theta^T_T \right\|^2_F,
\end{eqnarray}
\end{lemma1}
\begin{proof}
\begin{eqnarray}
\text{Dep}(Z_I, Z_T) 
&:=& 
\sum_{m=1}^r\sum_{j=1}^r \left\{\frac{1}{n}\sum_{i=1}^n f_j(\bm x_{I_i}) g_m(\bm x_{T_i})
 -\frac{1}{n^2}\sum_{p=1}^n f_j(\bm x_{I_p}) \sum_{k=1}^n g_m(\bm x_{T_k})\right\}^2\nn\\
&=& 
\sum_{m=1}^r\sum_{j=1}^r  \Big\{\frac{1}{n}\bm \theta_{I_j}^T \bm K_{X_I} \bm K_{X_T} \bm \theta_{T_m}
-\frac{1}{n^2} \bm \theta_{I_j}^T \bm K_{X_I} \bm 1_n \bm 1_n^T \bm K_{X_T} \bm \theta_{T_m} \Big\}^2 \nn\\
&=& 
\sum_{m=1}^r\sum_{j=1}^r \left\{ \frac{1}{n} \bm \theta_{I_j}^T \bm K_{X_I} \bm H \bm K_{X_T} \bm \theta_{T_m}\right\}^2 \nn\\
&=& 
\sum_{m=1}^r \frac{1}{n^2} \left\|\bm \Theta_I \bm K_{X_I} \bm H \bm K_{X_T}\bm \theta_{T_m}\right\|_2^2 \nn\\
&=& 
\frac{1}{n^2}\left\|\bm \Theta_I \bm K_{X_I} \bm H \bm K_{X_T} \bm \Theta_T^T \right\|^2_F
\end{eqnarray}
\end{proof}

\begin{theorem1}
Let the Cholesky factorization of $\bm K_X$ be $\bm K_X=\bm L_X \bm L_X^T$,  where $\bm L_X\in \mathbb R^{n\times d}$ ($d\le n$) is a full column-rank matrix. Let $r\le d$, then a solution to
{\footnotesize
\begin{equation}
    \label{eq:fixed-app}
    \max_{\bm f \in A_r} \{J\left(\bm f, \tau, \tau_z, \bm X, \bm Y, \bm S, \bm Z_O\right):=\frac{1}{n^2}\left\|\bm \Theta \bm K_{X} \bm H \bm L_Y \right\|^2_F - \tau \frac{1}{n^2}\left\|\bm \Theta \bm K_{X} \bm H \bm L_S \right\|^2_F + \tau_z \frac{1}{n^2}\left\|\bm{\Theta} \bm{K}_X \bm H \bm Z_O \right\|^2_F \}
\end{equation}}
is
\begin{eqnarray}
\bm f^{\text{opt}}(X) = \bm \Theta^{\text{opt}} \left[k_X(\bm x_1, X),\cdots, k_X(\bm x_n, X)\right]^T\nn
\end{eqnarray}
where $\bm \Theta^{\text{opt}}=\bm U^T \bm L_X^\dagger$ and the columns of $\bm U$ are eigenvectors corresponding to the $r$ largest eigenvalues of the following generalized eigenvalue problem.
{\footnotesize
\begin{eqnarray}
\label{eq:eigen}
\bm L_X^T\left(\bm H\bm K_Y\bm H  - \tau\bm H\bm K_S\bm H + \tau_z\bm H \bm Z_O \bm Z_O^T \bm H\right)\bm{L}_X\bm u = \lambda \left(\frac{1}{n}\,\bm L_X^T \bm H \bm L_X + \gamma \bm I\right) \bm u.
\end{eqnarray}}
Furthermore, the supremum value of the objective function is equal to $\sum_{j=1}^r\lambda_j$, where $\{\lambda_1,\cdots,\lambda_r\}$ are $r$ largest eigenvalues of~\eqref{eq:eigen}.
\end{theorem1}
\begin{proof}
Using the representer theorem, the disentanglement property in 
\begin{eqnarray}\label{eq:A-app}
\mathcal A_r:=\left\{\left(f_1,\cdots, f_r\right) \,\big|\, f_i, f_j\in \mathcal H_X,\, \cov \left(f_i(X), f_j(X) \right) +\gamma\, \langle f_i, f_j\rangle_{\mathcal H_X}=\delta_{i,j}\right\}.
\end{eqnarray}
can be expressed as
\begin{eqnarray}
&&\cov \left(f_i(X),\,f_j(X) \right) + \gamma\, \langle f_i, f_j \rangle_{\mathcal H_{X}}\nn\\
&=& \frac{1}{n}\sum_{k=1}^nf_i(\bm x_k) f_j(\bm x_k) -\frac{1}{n^2}\sum_{k=1}^n f_i(\bm x_k)\sum_{m=1}^n f_j(\bm x_m) + \gamma\, \langle f_i, f_j \rangle_{\mathcal H_{X}} \nn\\
&=&\frac{1}{n}\sum_{k=1}^n \sum_{t=1}^n \bm K_{X} (\bm x_k, \bm x_t)\theta_{i t}
\sum_{m=1}^n \bm K_{X} (\bm x_k, \bm x_m)\theta_{j m}-\frac{1}{n^2}\bm \theta_i^T \bm K_{X} \bm 1_n \bm 1_n^T \bm K_{X}\bm \theta_j+\gamma\, \langle f_i, f_j \rangle_{\mathcal H_{X}}\nn\\
&=& \frac{1}{n} \left(\bm K_{X} \bm \theta_i\right)^T
\left(\bm K_X \bm \theta_j\right)-\frac{1}{n^2}\bm \theta_i^T \bm K_{X} \bm 1_n \bm 1_n^T \bm K_{X}\bm \theta_j+\gamma\,
\left\langle \sum_{k=1}^n \theta_{ik}k_{X}(\cdot, \bm x_k), \sum_{t=1}^n \theta_{it}k_{X}(\cdot, \bm x_t)\right\rangle_{\mathcal H_{X}} \nn\\
&=& \frac{1}{n} \bm \theta_i^T \bm K_{X}\bm H \bm K_{X} \bm \theta_j
+ \gamma\,\bm \theta^T_i \bm K_{X} \bm \theta_j\nn\\
&=& \frac{1}{n} \bm \theta_i^T \bm L_{X} \left(\bm L^T_{X}\bm H \bm L_{X} + n\gamma\,\bm I\right) \bm L^T_{X} \bm \theta_j\nn\\
&=&\delta_{i,j}.\nn
\end{eqnarray}
As a result, $\bm f\in \mathcal A_r$ is equivalent to 
\begin{eqnarray}
\bm \Theta \bm L_{X} \underbrace{\Big( \frac{1}{n}\bm L^T_{X}\bm H \bm L_{X} + \gamma\bm I\Big)}_{:=\bm C} \bm L^T_{X} \bm \Theta^T= \bm I_r\nn,
\end{eqnarray}
where $\bm \Theta:=\big[ \bm \theta_1,\cdots, \bm \theta_r\big]^T\in \mathbb R^{r\times n}$.

Let $\bm V = \bm L_{X}^T\bm \Theta ^T $ and consider the optimization problem in~\eqref{eq:fixed}:
\begin{eqnarray}\label{eq:trace}
 &&  %
 \sup_{\bm f \in \mathcal A_r} \frac{1}{n^2}\left\{\left\|\bm \Theta \bm K_{X} \bm H \bm L_{Y} \right\|^2_F
 - \tau \left\|\bm \Theta \bm K_{X} \bm H \bm L_{S} \right\|^2_F
 + \tau_z \left\|\bm{\Theta} \bm{K}_X \bm H \bm Z_O \right\|^2_F\right\}\nn\\
  &=& %
  \sup_{\bm f \in \mathcal A_r } \frac{1}{n^2}\left\{\text{Tr}\left\{\bm \Theta \bm K_{X} \bm H \bm K_{Y} \bm H \bm K_{X}\bm \Theta^T\right\}
  - \tau  \text{Tr}\left\{\bm \Theta \bm K_{X} \bm H \bm K_{S}  \bm H \bm K_{X}\bm \Theta^T\right\}
 + \tau_z \text{Tr}\left\{\bm \Theta \bm{K}_X \bm H \bm Z_O\bm Z_O^T   \bm H \bm{K}_X\bm \Theta^T\right\}\right\}\nn\\
&=&\max_{\bm V^T \bm C \bm V = \bm I_r} \frac{1}{n^2} \text{Tr} \left\{\bm\Theta \bm L_{X}  \bm B \bm L_{X}^T \bm \Theta^T\right\}\nn\\
&=&\max_{\bm V^T \bm C \bm V = \bm I_r} \frac{1}{n^2} \text{Tr} \left\{ \bm V^T  \bm B \bm V \right\}
 \end{eqnarray}
where
\begin{eqnarray}
\bm B&:=&\bm L^T_{X}\left(\bm H \bm K_{Y} \bm H - \tau \bm H \bm K_{S} \bm H + \tau_z \bm H \bm Z_O\bm Z_O^T \bm H \right)\bm L_{X}\nn
\end{eqnarray}
It is shown in~\cite{kokiopoulou2011trace} that an\footnote{Optimal $\bm V$ is not unique.} optimizer of~(\ref{eq:trace}) is any matrix $\bm U$ whose columns are eigenvectors corresponding to $r$ largest eigenvalues of generalized problem
\begin{eqnarray}\label{eq:eig-gen-proof}
\bm B \bm u = \tau \,\bm C \bm u 
\end{eqnarray}
and the maximum value is the summation of $r$ largest eigenvalues. Once $\bm U$ is determined, then, any $\bm \Theta$ in which $\bm L_{X}^T\bm \Theta^T = \bm U$ is optimal $\bm \Theta$ (denoted by $\bm \Theta^{\text{opt}}$).
Note that $\bm \Theta^{\text{opt}}$ is not unique and has a general form of
\begin{eqnarray}
\bm \Theta^T = \left( \bm L_{X}^T\right)^\dagger \bm U + \bm \Lambda_0, \quad  \mathcal R(\bm \Lambda_0)\subseteq \mathcal N \left( \bm L^T_{X}\right).\nn
\end{eqnarray}
However, setting $\bm \Lambda_0$ to zero would lead to minimum norm for  $\bm \Theta$. Therefore, we opt $\bm \Theta^{\text{opt}}=\bm U^T \bm L_{X}^\dagger$.
\end{proof}

\subsection{Implementation Details \label{sec:imp-details}}
We conducted experiments on CelebA, Waterbirds, FairFace, and the Chicago Face Dataset (CFD). For CelebA and Waterbirds, we follow their official train/val/test splits and only use ground truth labels from the val split for hyperparameter tuning. For CFD, since there is no official dataset split, we randomly split it with a ratio of 0.5/0.1/0.4 for train/val/test. Following the standard setting \cite{zhang2022contrastive}, we use val split to decide the optimal $\tau$, $\tau_z$, and dimensionality of the random Fourier features (RFF). For CelebA, the optimal $\tau$, $\tau_z$, and RFF dimensions are 0.8, 0.5, and 8000. For Waterbirds, the optimal $\tau$, $\tau_z$, and RFF dimensions are 0.7, 0.7, and 3000. And for CFD, the optimal $\tau$, $\tau_z$, and RFF dimensions are 0.6, 0.3, and 1000. In the scenario where the group labels are not available, we follow the same setup as the scenario where the group labels of the val split are available. For the CelebA dataset, we also conduct a pre-sampling process on the training split to balance the number of each class from the predicted $\hat{y}$.

For the FairFace dataset, we use 10 text prompts that are unrelated to the facial attributes or the sensitive attributes of the samples. In this setting, the sensitive attribute is gender or race, and the text prompts are constructed as \emph{``This is a photo of a [attribute] person"} where \emph{[attribute]} can be one of the elements of the \{good, evil, smart, dumb, attractive, unattractive, lawful, criminal, friendly, unfriendly\} set.

\review{For all the above-mentioned experiments under different settings, we set the representation dimensionality $r$ to $c - 1$ where $c$ is the number of classes of the downstream target task.}

\review{For clarity, we summarized all the above-mentioned implementation details in \cref{tab:implementation-details}}
\review{
\begin{table}[h]
    \centering
    \caption{Implementation details of \methodName\ for each dataset.}
    \begin{tabular}{lccccccc}
    \toprule
    Dataset    & RFF Dim.  & $r$ & $\tau$ & $\tau_Z$ & Train/Val/Test & Training samples \\
    \midrule
    CelebA     &  8000 & 1 & 0.8 & 0.5 & Official Splits & 162,770\\
    Waterbirds &  3000 & 1 & 0.7 & 0.7 & Official Splits & 4,795\\
    FairFace   &  3000 & 1 & 0.8 & 0.8 & Official Splits & 86,744\\
    CFD        &  1000 & 1 & 0.6 & 0.3 & 0.5/0.1/0.4     & 298 \\
    \bottomrule
    \end{tabular}
    \label{tab:implementation-details}
\end{table}
}

\subsection{Numerical Complexity\label{sec:complexity}}
\noindent\textbf{Computational Complexity:} If $\bm L_X$ in~\eqref{eq:eigen} is provided in the training dataset, then the computational complexity of obtaining the optimal encoder is $\mathcal O(l^3)$, where $l\le n$ is the numerical rank of the Gram matrix $\bm K_X$. However, the dominating part of the computational complexity is due to the Cholesky factorization, $\bm K_X=\bm L_X \bm L_X^T$, which is $\mathcal O(n^3)$. Using random Fourier features (RFF)~\citep{rahimi2007random}, $k_X(\bm x, \bm x^\prime)$ can be approximated by $\bm r_X(\bm x)^T\, \bm r_X(\bm x^\prime)$, where $\bm r_X(\bm x)\in \mathbb R^{d}$. In this situation, the Cholesky factorization can be directly calculated as
\begin{equation}
    \bm L_X = \begin{bmatrix}
    \bm r_X(\bm x_1)^T\\
    \vdots\\
    \bm r_X(\bm x_n)^T
    \end{bmatrix}\in \mathbb R^{n\times d}.
\end{equation}
As a result, the computational complexity of obtaining the optimal encoder becomes $\mathcal O(d^3)$, where the RFF dimension, $d$, can be significantly less than the sample size $n$ with negligible error on the approximation $k_X(\bm x, \bm x^\prime)\approx\bm r_X(\bm x)^T\, \bm r_X(\bm x^\prime)$.

\noindent\textbf{Memory Complexity:}
The memory complexity of~\eqref{eq:eigen}, if calculated naively, is $\mathcal O(n^2)$ since $\bm K_Y$, $\bm K_S$, and $\bm Z_O \bm Z_O^T$ are $n$ by $n$ matrices. However, using RFF together with Cholesky factorization $\bm K_Y=\bm L_Y \bm L_Y^T$, $\bm K_S=\bm L_S \bm L_S^T$, the left-hand side of~\eqref{eq:eigen} can be re-arranged as
\begin{equation}
    \left(\bm L^T_X \tilde{\bm L}_Y\right) \left(\tilde{\bm L}^T_Y\bm L_X\right) 
    -\tau \left(\bm L^T_X \tilde{\bm L}_S\right) \left(\tilde{\bm L}^T_S\bm L_X\right) 
    + \tau_z \left(\bm L^T_X \tilde{\bm Z}_O\right) \left(\tilde{\bm Z}^T_O\bm L_X\right),
\end{equation}
where $\tilde{\bm Z}^T_O=\bm H \bm Z_O =\bm Z_O - \frac{1}{n}\bm 1_n (\bm 1_n^T\bm Z_O) $ and $\tilde{\bm L}^T_Y=\bm H \bm L_Y=\bm L_Y - \frac{1}{n}\bm 1_n (\bm 1_n^T\bm L_Y)$; therefore, the required memory complexity is $\mathcal O(nd)$. Note that $\tilde{\bm L}^T_S$ and $\bm H \bm L_X$ can be calculated similarly.

\subsection{Effect of Data Size on the performance of \methodName \label{sec:data-size}}
To evaluate the effectiveness of the \methodName\ under limited data samples condition we report the performance of \methodName\ as the size of the training dataset is varied. In this experiment, we randomly sampled 5, 25, 50, 75, and 100 percent of the training data as our training set. Then the Avg., WG, and Gap are evaluated. Table~\ref{tab:data-size} shows the results of the evaluation for the Waterbird and CelebA datasets. The results indicate that \methodName\ is able to perform sufficiently well when a small sub-sample of the dataset is used for its training. More specifically, in the CelebA dataset, \methodName\ is only losing $2.9\%$ of its WG accuracy when only $25\%$ of the original training data is employed in its training phase. Moreover, on Waterbirds, a relatively smaller dataset, it loses less than $6\%$ of WG accuracy when only $50\%$ of training data is used. This experiment shows the effectiveness of \methodName\ under a limited number of training samples.

\begin{table}[ht]
    \centering
    \caption{Effect of training data size on the performance of \methodName \label{tab:data-size}}
        \begin{tabular}{lccccccc}
            \toprule
            \multicolumn{1}{l}{\multirow{2}[4]{*}{\# Samples}} & \multicolumn{3}{c}{CelebA}  &&  \multicolumn{3}{c}{Waterbird} \\
                \cmidrule{2-4}  \cmidrule{6-8} &   Avg. ($\uparrow$)   & WG ($\uparrow$)    & Gap  ($\downarrow$) &&  Avg. ($\uparrow$)   &   WG ($\uparrow$)    & Gap  ($\downarrow$) \\
            \midrule
            $5\%$    &  84.21  &	73.88  &	10.32  &&  86.31  &	 65.26  & 21.05 \\
            $25\%$   &  88.26  &	83.88  &	4.37   &&  86.54  &	 78.66  & 7.88  \\
            $50\%$   &  86.65  &	84.29  &	2.36   &&  89.11  &	 81.31  & 7.80  \\
            $75\%$   &  90.44  &	85.56  &	4.88   &&  87.04  &	 84.11  & 2.92  \\
            $100\%$  &  89.3   &    86.7   &	2.6    &&  92.30  &	 87.70  & 4.60  \\
            \bottomrule
        \end{tabular}
\end{table}
\subsection{More Ablation Studies\label{sec:app:abl}}
In \ref{sec:ablation}, we studied the effect of different components of \methodName\ such as $\text{Dep}(Z,Y)$, $\text{Dep}(Z_I, Z_T)$, and updating the prediction of the target labels in each iteration. Here, we add another ablation study on the effect of our control hyper-parameters, $\tau$ and $\tau_z$ on the performance of the method. 

\textbf{Effect of  $\tau$ and $\tau_z$:} We illustrate the performance of average accuracy and worst group’s accuracy for varying values of $\tau$ and $\tau_z$ in Figure~\ref{fig:heatmap}. First, as $\tau$ and $\tau_z$ vary, there is a smooth change in the average and worst group accuracy, which demonstrates the stability of \methodName. Second, as $\tau$ increases, the worst group’s accuracy also improves, which alludes to the effectiveness of $\tau$ as a control for the degree of debiasing. However, when $\tau$ reaches a certain value, a further increase in its value leads to a degradation in the worst group’s accuracy.
 Similarly, $\tau_z$ also plays a gradual but noticeable effect on improving both the average and worst group accuracy for a given value of $\tau$.
 
\begin{figure}
    \centering
    \includegraphics[width=0.8\textwidth]{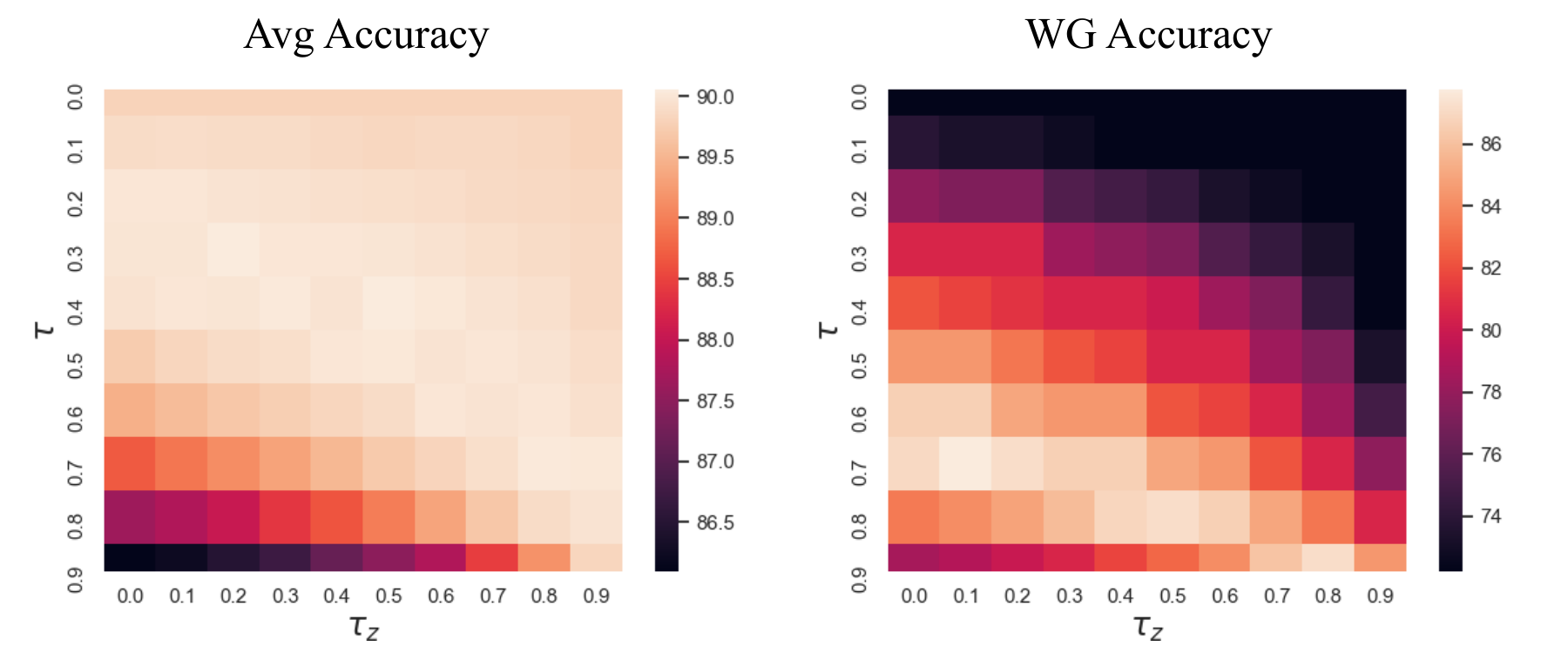}
    \caption{Effect of $\tau$ and $\tau_z$ on the average accuracy (left) and worst group accuracy (right). For every fixed $\tau_z$, the worst group accuracy increases with $\tau$ and then starts to decrease beyond a point. Additionally, the variation of the metrics is smooth over the range of hyperparameters, indicating lower sensitivity towards them.\label{fig:heatmap}}
\end{figure}

\subsection{Comparing Features of CLIP ViT-L/14 and CLIP ResNet-50 \label{sec:app:information}}
\begin{table}[ht]
    \centering
    \caption{Comparison between CLIP ViT-L/14 and CLIP ResNet-50 in terms of the amount of the information from $Y$ and $S$ that each of them can embed into their feature space on the Waterbirds.\label{tab:information}}
        \begin{tabular}{lccccc}
            \toprule
             \multicolumn{1}{l}{\multirow{2}[4]{*}{}} & \multicolumn{2}{c}{($X, Y$)}  &&  \multicolumn{2}{c}{($X, S$)} \\
                \cmidrule{2-3}  \cmidrule{5-6} &   HSIC   & KCC    && HSIC &  KCC \\
            \midrule
            CLIP ViT-L/14    &  0.1849  &	0.8267  &&	0.2392   &  0.8661 \\
            CLIP ResNet-50   &  0.1423  &	0.7556  &&	0.3823   &  0.8861 \\
            \bottomrule
        \end{tabular}
\end{table}

To compare the features of CLIP ViT-L/14 and CLIP ResNet-50, we measure the amount of information from the target attribute and sensitive attribute contained in their generated representations. The embedded information is measured in terms of statistical dependency between the features and their ground-truth $Y$ and $S$ labels. To calculate these dependencies, Hilbert-Schmidt independence criterion (HSIC) \citep{gretton2005kernel} and  Kernel Canonical Covariance (KCC) \cite{bach2002kernel} are used. \cref{tab:information} compares these two CLIP models. From the table, we can observe that CLIP ViT-L/14 embeds more information about the target $Y$ while containing less information about $S$ which indicates that the former provides better features for the Waterbirds dataset.

\subsection{Comparing more than 100 Zero-Shot CLIP models on CFD \label{sec:app:cfd}}

As we mentioned in \Cref{sec:exp:results:spurious}, zero-shot classification on CFD is a difficult task for the OpenAI CLIP model. In \Cref{fig:clip-cfd}, we show that a majority of other publicly available CLIP models suffer similarly on CFD. In fact, several models achieve only near-zero WG accuracy, irrespective of their parameter count and the training dataset.

\begin{figure}[ht]
    \centering
    \includegraphics[width=0.9\columnwidth]{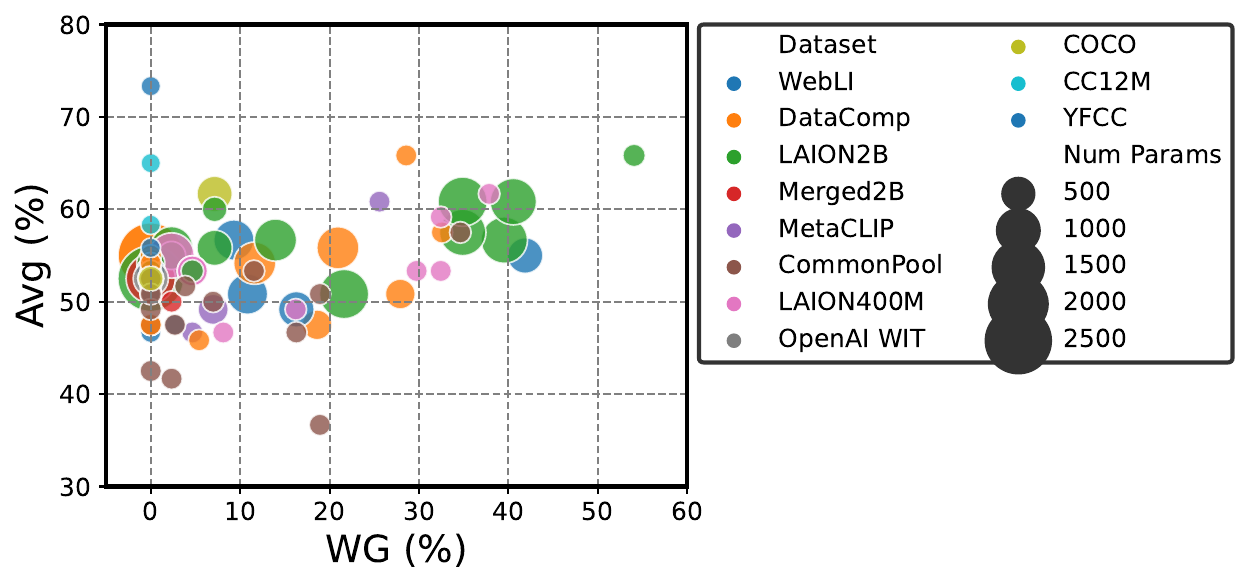}
    \caption{Comparison of more than 100 publicly available CLIP models zero-shot performance on CFD dataset. Colors show the pre-trained dataset and sizes show the number of parameters of each model.\label{fig:clip-cfd}}
\end{figure}

\end{document}